\documentclass[10pt,twocolumn]{article}
\usepackage[letterpaper,margin=0.75in,columnsep=0.28in]{geometry}
\usepackage[T1]{fontenc}
\usepackage{lmodern,microtype,amsmath,amssymb,amsthm,booktabs,graphicx,xcolor}
\usepackage[numbers,sort&compress]{natbib}
\usepackage[hyphens]{url}
\usepackage[colorlinks=true,allcolors=teal]{hyperref}
\hypersetup{pdftitle={The Delegation Blind Spot: Auditing Product Decisions from Agent Choices},pdfauthor={Shivam Gupta},pdfsubject={Decision-specific measurement of delegated choices; computational technical preprint}}
\usepackage{enumitem,fancyhdr}
\setlist{nosep,leftmargin=*}
\newtheorem{theorem}{Theorem}
\newtheorem{proposition}[theorem]{Proposition}

\theoremstyle{definition}
\newcommand{\simplex}{\Delta_K}
\newcommand{\R}{\mathbb R}

\newcommand{\row}{\operatorname{row}}
\newcommand{\1}{\mathbf 1}
\newcommand{\norm}[1]{\left\lVert#1\right\rVert}
\title{\vspace{-0.6cm}\textbf{The Delegation Blind Spot}\\[4pt]
\large Auditing Product Decisions from Agent Choices}
\author{Shivam Gupta\\Independent Research\\\small\href{mailto:shivam1720406@gmail.com}{shivam1720406@gmail.com}}
\date{20 September 2026}

\begin{document}
\raggedbottom
\maketitle
\begin{abstract}
Successful agent execution need not identify which future product improvement its user would value. We present a decision-specific audit that maps a declared observation channel and product-value contrast to compatible intervals and witness populations. Its foundations are established identification and decision theory; the contribution is an executable measurement workflow and a controlled study of its limits. A frozen experiment makes 4,800 requests to two pinned model snapshots on shared synthetic tasks. All 36 conservative primary intervals remain unresolved despite different execution accuracy. An exploratory 2,400-call follow-up records supplied preferences and resolves three of nine comparisons per model. A deterministic extractor resolves seven of nine without model calls or calibration observations, exposing unnecessary uncertainty introduced by model-generated reports. A further 14,400 controlled multinomial simulations distinguish structural ambiguity from weak identification and finite calibration precision. We propose a source-labeled decision receipt and provide an offline viewer for inspecting the audit. These results motivate preserving decision-relevant structured input and diagnosing why a decision is unresolved before collecting more telemetry. The study contains no human participants or real customer outcomes. Full proofs, raw model provenance, controlled experiments, and reproducible analyses accompany the report.
\end{abstract}

\section{Introduction}
Two customers ask their agents to book the same hotel. One values quiet; the other values proximity to a meeting. At current prices, the same hotel is best for both. Both agents execute perfectly, and the platform records indistinguishable transactions. The platform's next question is different: should it invest in soundproofing or transport links? Successful bookings need not distinguish the populations that would favor those investments.

This ambiguity is possible with human choices as well. We do not claim that delegation creates nonidentification, or that more capable agents necessarily remove information. Instead, we study a precise operational question: \emph{does a specified stream of delegated activity support a specified future product decision?} This requires distinguishing execution utility, information available to the company, and the customer's value from a future change.

Our contribution is a reproducible decision-oriented audit. It reports bounds and competing compatible populations, then separates reasons the evidence may remain inconclusive. Structural ambiguity can require a different observation; sampling uncertainty can require better calibration or more data. These interventions are not interchangeable. The computational evaluation combines actual model responses, a deterministic extraction baseline, and controlled channels with analytically known identification properties.

The audit operationalizes existing identification principles for a declared future product contrast. Its mathematical results and linear programs have direct antecedents. The study evaluates this workflow on constructed preferences and utilities; it does not establish market prevalence, customer adoption, or commercial value.

\section{Related work}
Blackwell's comparison of experiments connects observations to decision value \citep{blackwell1953}. Linear partial monitoring makes the relation between payoff contrasts and observation spans explicit \citep{kirschner2023}. Optimal recovery relates inverse estimation to distances and moduli of continuity \citep{donoho1994}. These supply the conceptual basis for our observability and lower-bound results.

The channel equation used here is closely related to label-shift estimation: a calibrated confusion matrix maps latent class prevalence to observable predictions \citep{lipton2018}. Full prevalence recovery can require invertibility, whereas a particular contrast may remain identifiable without full recovery. Partial identification under measurement error provides linear-programming bounds and confidence-set propagation \citep{finkelstein2021}. We use these established principles and expose their assumptions to product decisions.

Delegated choice itself has been studied through rationalization and menu interpretation \citep{kops2026}, consumer preference transmission \citep{kraft2026}, and controlled tests of agent responses to product attributes and stated profiles \citep{cherep2025}. Revealed-preference analysis can identify aspects of human-agent alignment from choices across menus under a mixture-of-Luce model \citep{suleymanov2026}. That target and its menu variation differ from our calibrated population contrast and coarse logs. Unresolved intervals here do not refute identification from richer observations. Synthetic profiles are not equivalent to the human evidence in preference-transmission studies. Prediction-powered inference and active sampling address efficient acquisition of true outcome labels \citep{ppi2023,asi2024,robust2025}; our initial audit pilot uses those methods as antecedents and controls.

\section{A decision-specific observation model}
Let $p\in\simplex=\{p\in\R^K:p\geq0,\1^Tp=1\}$ denote a population's proportions across a declared finite intent taxonomy. A fixed agent, interface, and task distribution induce an $m\times K$ column-stochastic channel $A$, with
\begin{equation}
 A_{ak}=\Pr(Y=a\mid Z=k),\qquad q=Ap.
\end{equation}
Here $Y$ is a logged action category. Its definition is part of the measurement design. Logging a selected product type, logging type together with context, and logging the entire continuous menu are different experiments.

For product variants 1 and 0, let $d_k$ be their independently specified difference in customer outcome for class $k$. The target is
\begin{equation}\Delta(p)=d^Tp.\end{equation}
This is a value contrast, not automatically a causal effect or financial return. Its causal interpretation requires an appropriate outcome design; an investment decision additionally requires costs and constraints. In the computational study, $d$ is a known mean over an independent finite bank of synthetic menus. We analyze uncertainty in estimated $d$ separately.

The compatible population set is $\mathcal P_q=\{p\in\simplex:Ap=q\}$. Define $L(q)=\min_{p\in\mathcal P_q}d^Tp$ and $U(q)=\max_{p\in\mathcal P_q}d^Tp$. Compactness guarantees attainment. A positive lower endpoint supports variant 1 throughout the supplied model; a negative upper endpoint supports variant 0. A straddling interval supplies witness populations favoring opposite decisions. An interval touching zero may represent a tie, not a strict reversal.

\section{Identification, ambiguity, and decision loss}
The following results are stated in our notation for auditability. Full proofs and further extensions appear in Appendix~\ref{sec:proofs}.

\begin{theorem}[Global and local identification]\label{thm:main-id}
For known $A$, the contrast is identified for every feasible $q$ if and only if $d\in\row(A)$. At a fixed $q$, identification holds if and only if $d$ is orthogonal to $\operatorname{span}(\mathcal P_q-\mathcal P_q)$.
\end{theorem}
\begin{proof}[Proof sketch]
If $d=A^Tv$, then $d^Tp=v^Tq$. Otherwise a null vector $h$ with $Ah=0$ and $d^Th\ne0$ exists. Column normalization gives $\1^Th=0$; perturbing an interior population in directions $\pm h$ constructs indistinguishable populations. The local statement follows from constancy over feasible differences.
\end{proof}
The quantifiers matter. A rank-deficient channel can identify all contrasts at a boundary observation that isolates one class. Conversely, reconstructing all class proportions is stronger than learning one particular product contrast.

\begin{theorem}[Worst compatible contrast width]\label{thm:main-width}
Let $W(A,d)=\max_{p,r\in\simplex:Ap=Ar}|d^T(p-r)|$. Then
\begin{align}
 W(A,d)&=\max_{Ah=0,\norm{h}_1\leq2}d^Th\\
 &=2\min_v\norm{d-A^Tv}_\infty.
\end{align}
\end{theorem}
Every feasible difference has zero total mass and $\ell_1$ norm at most two. Conversely, its positive and negative parts can be completed with a common population mass. Linear-programming duality gives the distance expression. This is a standard inverse-problem quantity specialized to a decision contrast, not a new duality principle. $W$ is global: the interval at an actual $q$ can be narrower.

For action $b\in\{0,1\}$, value regret is $\ell(b,\Delta)=\max(0,\Delta)-b\Delta$.
\begin{theorem}[Irreducible decision loss]\label{thm:main-risk}
Fix an exactly known $q$ with attainable interval $[L,U]$. Every estimator based only on any number of iid draws from $q$ has worst compatible mean absolute error at least $(U-L)/2$ and mean squared error at least $(U-L)^2/4$. If $L<0<U$, the exact-observation minimax randomized value regret is
\begin{equation}
 R^*(q)=\frac{-LU}{U-L},\qquad
 t^*(q)=\frac{U}{U-L},
\end{equation}
where $t^*$ is the probability of choosing variant 1. Outside the straddling case, minimum regret is zero.
\end{theorem}
Endpoint populations generate the same observation law. Estimation lower bounds follow from triangle and squared-loss identities; balancing endpoint regrets $(-L)t$ and $U(1-t)$ proves the decision result. Randomization is a theoretical decision rule, not a recommendation to randomize a costly business commitment. Deterministic minimax regret is $\min(-L,U)$ in the straddling case.

\paragraph{Example.} With identical columns $A_{\cdot1}=A_{\cdot2}=(0.8,0.2)^T$ and $d=(-0.3,0.5)^T$, all populations have the same action law. The interval is $[-0.3,0.5]$ and $R^*=0.1875$. More samples from this channel cannot resolve it. Direct outcome feedback or a new informative channel can.

\section{Uncertainty and measurement design}
\subsection{Attainable bounds with uncertain channels}
Suppose $0\leq\underline A\leq\overline A\leq1$, every channel-column box intersects the probability simplex, and field frequencies lie in $[\underline q,\overline q]$. Introduce joint masses $J_{ak}=A_{ak}p_k$ and impose
\begin{align}
 &p\in\simplex,\quad J\geq0,\quad \sum_aJ_{ak}=p_k,\\
 &\underline A_{ak}p_k\leq J_{ak}\leq\overline A_{ak}p_k,\\
 &\underline q_a\leq\sum_kJ_{ak}\leq\overline q_a.
\end{align}
Minimize and maximize $d^Tp$ over this linear feasible set. The reparameterization is exact for the declared rectangular model: divide each positive-mass column of $J$ by $p_k$ to recover its channel; choose any feasible channel column when $p_k=0$. Every original model maps back to such a joint mass. Sharpness is relative to this uncertainty description, not to unspecified correlated constraints.

The implementation verifies solver witnesses against the constraints and returns failure for infeasibility. Empty feasible sets must not become confident recommendations. If the channel and field sets cover with error probabilities $\alpha_A$ and $\alpha_q$, the target interval covers with probability at least $1-\alpha_A-\alpha_q$, conditional on model correctness. Independence of these coverage events is unnecessary for the union bound. We use simultaneous coordinate Clopper--Pearson boxes within each channel/field calculation. Their iid sampling assumptions must match the experiment.

If $d\in[\underline d,\overline d]$, optimize $\underline d^Tp$ for the lower endpoint and $\overline d^Tp$ for the upper endpoint. Nonnegativity of $p$ makes this exact for a rectangular outcome set. Its failure probability adds to the preceding bound. Unknown classes and channel drift are not automatically covered by these sampling intervals. Separately constructed coordinate boxes for coarse and fine logs need not be nested. Thus finite-sample interval widths across schemas are not ordered by the exact-channel coarsening theorem.

\subsection{A bias-aware linear certificate}
For a known reference channel, select action weights $v$ before observing the field sample. Write $\epsilon=\norm{d-A^Tv}_\infty$ and $s=\max_a v_a-\min_a v_a$. For $n$ iid observations from a deployed channel whose class columns are within total variation $\tau$ of the reference,
\begin{equation}
 \left|\frac1n\sum_i v_{Y_i}-d^Tp\right|
 \leq \epsilon+s\left(\tau+\sqrt{\frac{\log(2/\alpha)}{2n}}\right)
\end{equation}
with probability at least $1-\alpha$. The terms separate approximation, channel drift, and sampling error. Our code minimizes this bound by linear programming independently of field observations. Estimated calibration must be covered by $\tau$ or another uncertainty argument; substituting a point-estimated channel is not a coverage guarantee.

\subsection{Record the measurement condition}
Suppose a probe $e$ is independently assigned with probability $\rho_e>0$, logged, and preserves population composition. Its joint channel is $B=\operatorname{vstack}(\rho_eA_e)$. It identifies $d$ exactly when $d$ belongs to the joint row span. Discarding the probe label replaces this with $\sum_e\rho_eA_e$ and can destroy information. For example, equally likely identity and row-swapped identity channels are fully informative when the probe is recorded but completely pooled when it is omitted. This is an implication of coarsening and classical experiment comparison, not evidence that any particular agent update is a coarsening.

\section{Computational experiment}
\subsection{Frozen design and provenance}
The main design was committed and pushed before its API calls. It is an exploratory, prospectively frozen computational study, not an externally preregistered human experiment. A separate one-call connectivity test is excluded. The main run queries GPT-5.4 Mini and GPT-5.4 Nano, both pinned to their 17 March 2026 snapshots, on identical task records, interleaving model requests to reduce timing confounds. Prompts, exact requests, response identifiers, returned model versions, token usage, errors, and source hashes are retained. There are no tools or adaptive prompt changes during the run.

Three domain framings (travel, cloud plans, workflow software) share a mathematical generator. Each task has four alternatives with independently perturbed attributes, one of three context regimes, and randomized option order. Four declared intent classes use explicit numerical weights and a noncompensatory minimum-attribute penalty. Models choose an option ID under a strict output schema. These are synthetic utility objectives, not elicited human preferences; domain framings are not independent real-world datasets.

For each model and domain, calibration samples 80 tasks per intent class. Three independently sampled field cohorts (160 tasks each) have specified class mixtures: positive $(.65,.10,.15,.10)$, negative $(.15,.10,.65,.10)$, and near $(.42,.18,.27,.13)$. Thus each model is assigned 2,400 requests, totaling 4,800. Successful responses and transport failures are reported separately. Separate random streams generate calibration, field cohorts, future-outcome banks, and outcome samples. The target contrasts compare best attainable utility after capacity versus flexibility investment, each with an identical affordability tradeoff, over 4,096 independent synthetic menus per domain. This assumes optimal future selection; it is not measured utility of the tested model under those future variants.

\subsection{Targets, schemas, and baselines}
The primary target uses exact mean class contrasts from that finite outcome bank. It isolates observation-channel measurement from outcome estimation. A secondary analysis uses 2,048 samples from the bank and bounded-outcome uncertainty. Neither establishes real customer value.

We compare two declared logs: selected semantic product type (plus failure), and type paired with the three context regimes (plus failure). Both omit full continuous menu information. Unresolved intervals under these schemas are not proof of nonidentification from every possible observable feature.

Baselines include constrained least-squares prevalence estimation inspired by label-shift correction, a percentile bootstrap conditional on the estimated channel, and the joint-mass uncertainty diagnostic. The fixed-channel bootstrap omits calibration uncertainty and has no claimed finite-sample coverage. Exact-optimum, uniform-random, and fixed-default choosers are algorithmic controls, not additional models or participants. We measure task success, synthetic utility regret, interval width, decision resolution, compatibility with the known synthetic target, and provider failures.

Primary intervals allocate $\alpha_A=\alpha_q=.02$, giving at least 96\% marginal coverage per contrast under the stated assumptions. Secondary outcome uncertainty adds $.01$, giving 95\%. These are not simultaneous guarantees across the model/domain/cohort/schema comparisons. We report all conditions and do not treat their observed coverage fraction as an independent repeated-sampling validation of the theorem.

\subsection{Exploratory explicit-preference receipts}
After observing that all primary intervals were unresolved, we froze a separate follow-up before its new API calls. It asks the agent to report the largest of the four preference weights already present in its input. Selection uses the first 40 lexicographic calibration task IDs per domain/class and the first 80 field IDs per domain/cohort, without consulting their response correctness or utility. This yields 1,200 new requests per model. The original actions for exactly those task IDs provide an equal-count baseline. The primary follow-up comparison is attribute-only reporting against action-only logging, with nine domain/cohort conditions per model.

The receipt removes the offer menu. Its 1,200 calls per model repeat only 12 distinct prompts: four supplied weight profiles in three framings. Each profile has a unique leading attribute, so an accurate receipt directly reveals its synthetic class. It tests whether this explicit report produces a useful observation channel in the constructed problem. It does not test inference of private preferences or independently validate the supplied weights. A deterministic parser could recover these weights directly; a language model is not necessary for that operation in a deployed system. Matched observation counts do not equalize tokens, monetary cost, privacy exposure, or human elicitation burden.

The snapshots are unchanged. Receipt requests explicitly set reasoning effort to none; the primary requests use the documented default of none. The prompt and task change, rather than isolating a passive logging intervention. Repeated calls require a stable, independent response mechanism for the categorical sampling model. Moreover, the follow-up was chosen after the primary results and reuses original field draws and actions. Its fixed-design marginal interval statements are not selective guarantees conditional on that choice. Resolution counts are exploratory descriptions, not a confirmatory significance test.

\subsection{AI-assisted research development and verification}
OpenAI Codex was used in an iterative author-directed workflow for problem formulation, literature retrieval, proof drafts, experiment protocols, code, analysis, visualizations, and manuscript text. Internal AI critiques informed revisions; they are not external peer review. Verification includes primary-source comparison, mathematical counterexamples, numerical constraint checks, automated tests, recorded API provenance, and exact reanalysis. These automated and source-based checks do not certify human-author review. The author is responsible for scrutinizing the work and its claims before submission. Research assistance is distinct from the two pinned models evaluated in the experiments.

\section{Results}
\subsection{Execution competence and decision uncertainty}
The primary run recorded 4,800 attempts, 4,783 valid choices, and 17 transport failures. Both returned model identifiers match their pinned requested snapshots. Failed records are retained as an observable category. No retries replace them.

Table~\ref{tab:model-performance} reports field-task performance, counting a failure as no delivered utility. Mini chooses a utility maximizer more often and incurs lower mean synthetic regret. These are descriptive results for the frozen generator, not a general model ranking. Calibration tasks are excluded from these field-performance denominators.

\begin{table}[t]
\centering\small
\begin{tabular}{@{}lrr@{}}
\toprule Field metric & Mini & Nano \\\midrule
Tasks & 1,440 & 1,440 \\
Utility-maximizing choice & 73.19\% & 55.69\% \\
Mean synthetic regret & 0.0186 & 0.0573 \\
Field delivery failures & 0.35\% & 0.28\% \\\bottomrule
\end{tabular}
\caption{Real model responses to constructed utility objectives. Regret is best available utility minus delivered utility, on the declared $[0,1]$ scale.}
\label{tab:model-performance}
\end{table}

The joint-mass diagnostic leaves all 18 conditions per model unresolved. It contains each known synthetic target, but this does not establish an empirical coverage rate: conditions share calibration data, and the set is small and selected. Zero wrong certified decisions is achieved here by abstaining on every condition. Even the optimal-choice control leaves all 18 corresponding intervals unresolved. At this budget, conservative coordinate boxes are not a practically sufficient product-selection procedure.

The constrained point baseline selects the wrong sign in 2 of 18 Mini conditions and 4 of 18 Nano conditions. The fixed-channel percentile bootstrap contains the known target in 14 of 18 and 16 of 18 conditions, respectively. These are descriptive counts, not a calibrated coverage comparison. The bootstrap omits calibration error; the conservative method includes it but sacrifices decisiveness. This experiment does not isolate intrinsic channel nonidentification from finite-sample uncertainty.

\begin{figure*}[t]
\centering\includegraphics[width=.94\textwidth]{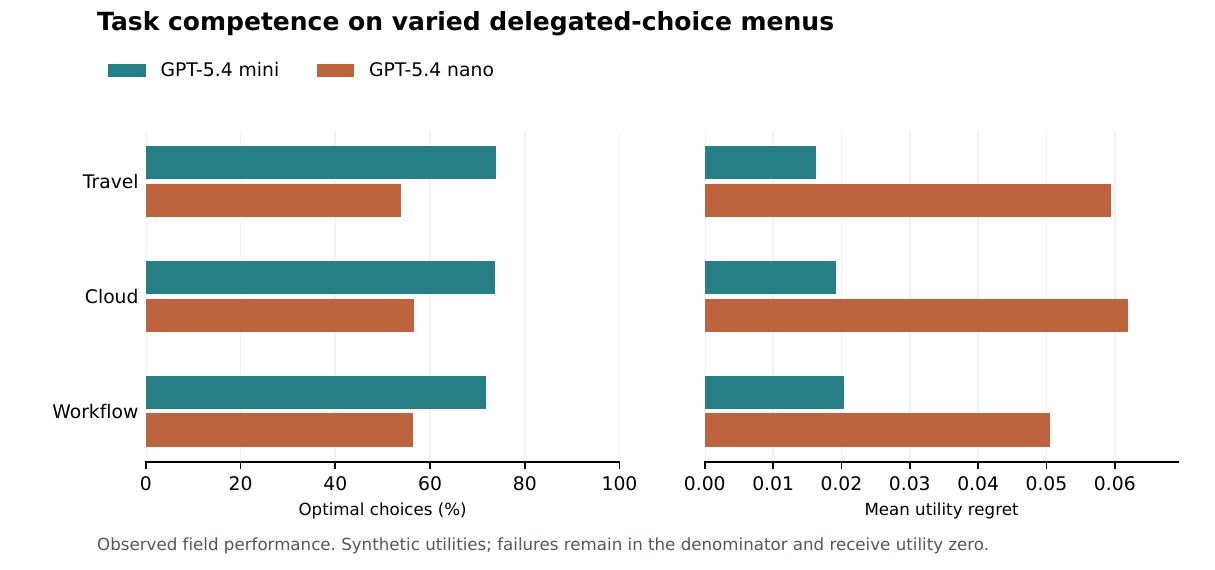}
\caption{Primary field-task performance by semantic framing. The three framings share a mathematical generator. Real provider failures remain in the operational scores.}
\label{fig:performance}
\end{figure*}

\begin{figure*}[t]
\centering\includegraphics[width=.95\textwidth]{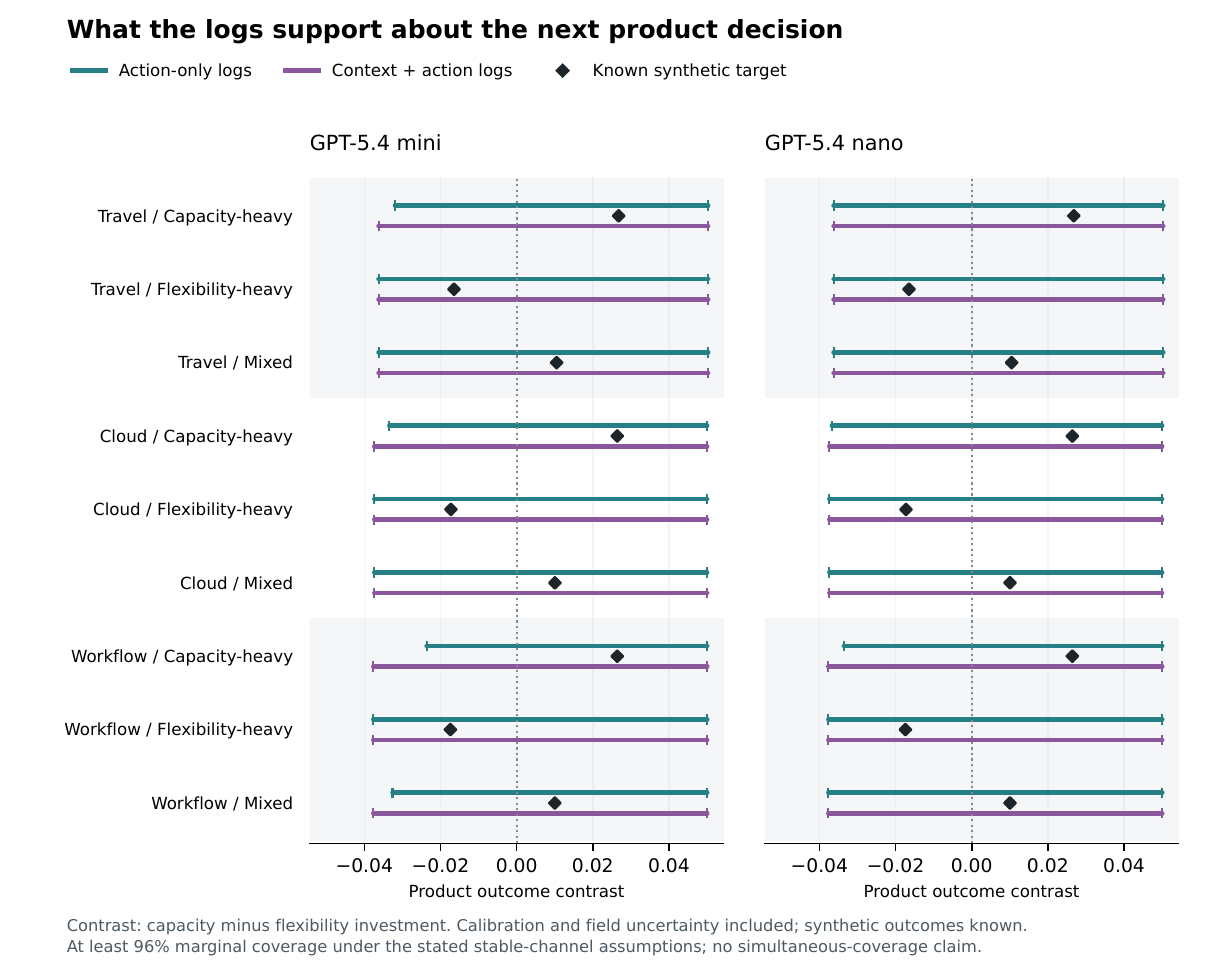}
\caption{Primary decision intervals for the known finite-bank synthetic contrast. Each row uses the same independent calibration within a domain. Markers show the evaluation-only true target. All intervals cross zero. Nominal marginal coverage is at least 96\% under the stated model and sampling assumptions; this is not a simultaneous guarantee.}
\label{fig:intervals}
\end{figure*}

\subsection{Compute and interpretation}
Recorded tokens imply an estimated US\$1.113 for the primary run at the documented standard rates. This is not an invoice: failed transport requests may have unreported billable usage. Timeouts and connection resets are not evidence of model reasoning errors. Their temporal clustering also cautions against treating every operational error as an independent stationary draw. Conditional mathematical coverage statements must not be mistaken for verified provider independence.

The secondary analysis also propagates synthetic outcome estimation error and is more conservative. Exact outcome-bank values describe best attainable utility after future investment, not realized future behavior of the tested models. The primary findings justify investigating additional measurement channels, not declaring that all delegated product activity is inherently uninformative.

\subsection{Explicit preference information at matched counts}
The exploratory follow-up recorded 2,400 completed receipt requests, with no delivery failures. Both models returned their requested snapshots and correctly reported the largest supplied weight on every calibration and field request. This is accuracy on 12 repeated, explicitly specified prompts, not evidence of general preference understanding. The two models therefore induce identical empirical receipt channels and identical receipt intervals on the shared task subset.

Each model's receipt channel resolves three of nine primary follow-up conditions, compared with zero of nine using the matched original actions. All three resolved conditions are the positive cohort, one per domain; none contradicts the known synthetic target. The negative and near cohorts remain unresolved. Mean interval width decreases from 0.0872 for action logs to 0.0436 for receipts, a 50.0\% reduction. These figures summarize the selected conditions and do not constitute a significance test or a post-selection coverage guarantee.

\begin{table}[t]
\centering\small
\begin{tabular}{@{}lrr@{}}
\toprule Paired measure, per model & Actions & Receipts \\\midrule
Calibration per class/domain & 40 & 40 \\
Field per cohort/domain & 80 & 80 \\
Resolved contrasts / conditions & 0 / 9 & 3 / 9 \\
Incorrect resolved contrasts & 0 & 0 \\
Mean interval width & 0.0872 & 0.0436 \\\bottomrule
\end{tabular}
\caption{Exploratory comparison. Both models have the same reported interval summaries. Perfectly reporting the supplied leading attribute directly reveals the constructed intent class. Counts do not equalize monetary cost or elicitation burden.}
\label{tab:receipt}
\end{table}

\begin{figure*}[t]
\centering\includegraphics[width=.94\textwidth]{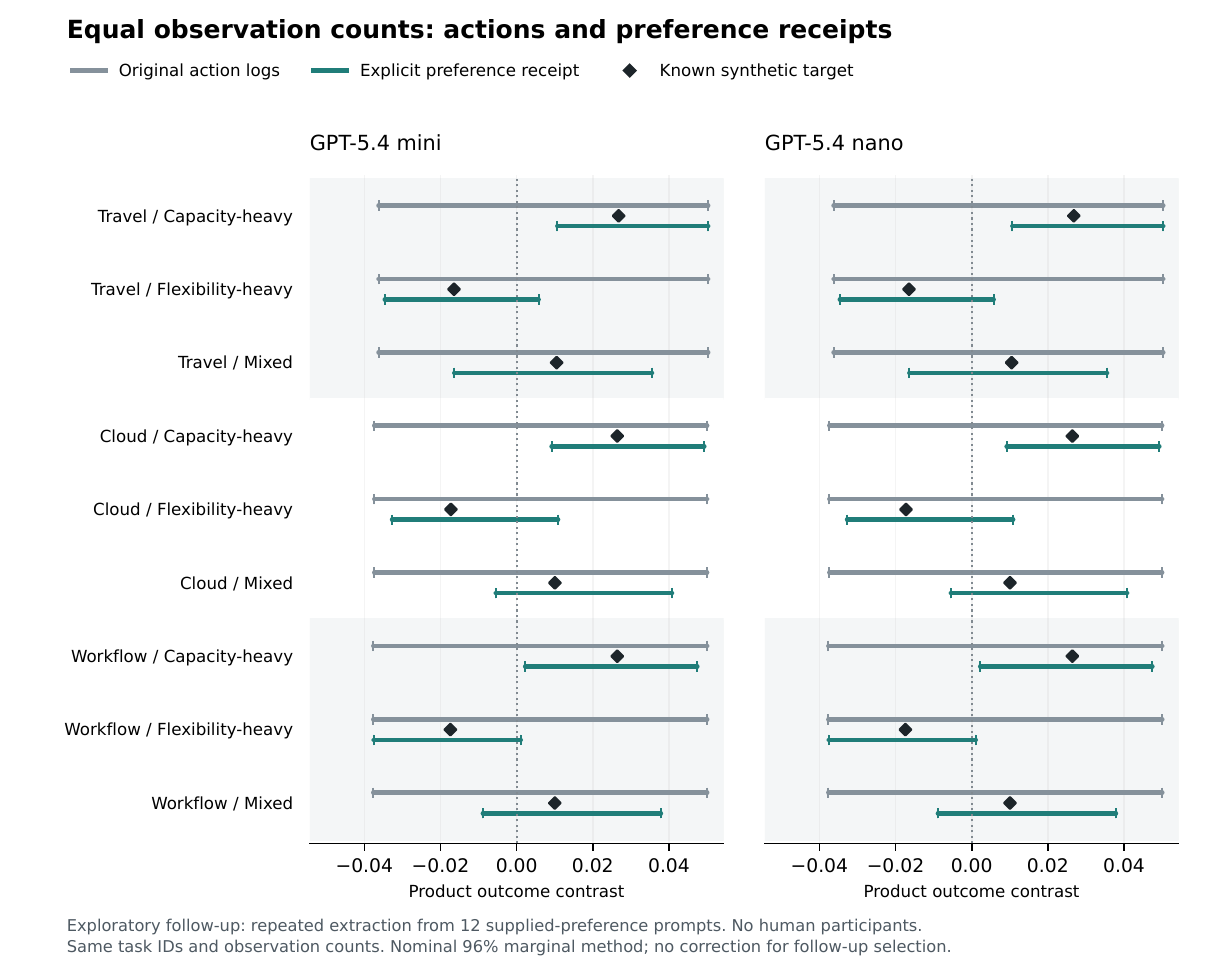}
\caption{Action-only and attribute-only intervals on the same selected task IDs. The receipt collects explicit preference information omitted from the action log; it does not discover an unobserved human preference. The follow-up was chosen after the primary results, so these are exploratory comparisons without selective-inference correction.}
\label{fig:receipt}
\end{figure*}

The result separates two limitations. Even a directly observed class does not resolve every small contrast under this finite-sample procedure. Yet the same number of observations is more useful when it records a relevant class attribute instead of a behavioral proxy. The receipt's remaining uncertainty is sampling and channel-calibration uncertainty within the declared model, rather than a demonstrated loss of the supplied class information. Since a deterministic extractor would provide the same field, this experiment supports explicit measurement design, not the necessity of an LLM receipt generator.

Reported usage implies an additional US\$0.372 for the follow-up. Across the two studies, there are 7,200 attempted API requests, 7,183 valid responses, and 17 retained transport failures. The excluded connectivity test is not part of these counts. All preferences and product outcomes remain synthetic.

\subsection{A deterministic baseline for supplied preferences}
The receipt experiment asks models to recover information already in a structured input. A new offline control parses those supplied weights directly and checks membership in the declared four-profile taxonomy. It uses the same 720 field records across the nine conditions, without their private intent labels. The known identity channel eliminates report-calibration uncertainty. With $\alpha_q=.04$, field-frequency bounds resolve seven of nine comparisons, with zero incorrect resolutions and mean width 0.02285. There are no calibration observations or new API calls. A direct Hoeffding interval for the bounded variable $d_Z$, using the same $.04$ error budget, also resolves seven; its mean width is 0.02727. The methods are reported separately, not intersected.

This baseline changes the engineering interpretation of the receipt result. Given these structured inputs, routing the supplied attribute through an LLM and estimating its error channel is unnecessary. Preserving the field is simpler and produces more decisive intervals in this selected benchmark. It still does not validate the supplied profile against a real customer. These reanalyses use the previously selected field draws and remain exploratory.

\subsection{Structural ambiguity versus finite precision}
A separate controlled simulation isolates limitations conflated by the primary model experiment. Let
\begin{equation}
 A_\eta=(1-\eta)\mathbf1\mathbf1^T/4+\eta I_4,
 \quad d=(.06,.02,-.04,-.02)^T.
\end{equation}
The channel eigenvalues are $1,\eta,\eta,\eta$. At $\eta=0$, every population has the same action law and the compatible contrast interval is $[-.04,.06]$. For every $\eta>0$, $A_\eta$ is invertible and the exact-observation interval has zero width. Identification alone therefore does not describe the difficulty of the finite-sample inverse problem.

For known $A_\eta$ and $\eta>0$, write $\bar d=\mathbf1^Td/4$. The weights $v=\bar d\mathbf1+(d-\bar d\mathbf1)/\eta$ satisfy $A_\eta^Tv=d$, with range $\operatorname{range}(d)/\eta$. The certificate in Section~5.2 gives error at most
\begin{equation}
 \frac{\operatorname{range}(d)}{\eta}
 \sqrt{\frac{\log(2/\alpha)}{2n}}.
\end{equation}
Thus $n\geq\operatorname{range}(d)^2\log(2/\alpha)/(2\eta^2\epsilon^2)$ suffices for radius $\epsilon$ with these fixed weights. This is a sufficient Hoeffding bound, not an optimal sample-complexity or impossibility claim. Channel estimation adds another uncertainty source.

The sweep uses $\eta\in\{0,.01,.03,.1,.3,1\}$, calibration counts per class $c\in\{40,160,640,2560\}$, field counts $2c$, and the three original population mixtures. Each of 72 cells has 200 independent complete repetitions, totaling 14,400 multinomial simulations and 43,200 interval calculations. The three methods use exact $A$ with uncertain field frequencies, uncertain $A$ with exact $q$, or uncertainty in both. The first two have 98\% marginal guarantees, and the joint procedure has 96\%; these are uncertainty-source ablations, not an equal-confidence ranking. Methods share samples within a repetition; repetitions and cells use independent streams. All outcomes are constructed, and no new model or human observations are generated.

For the positive mixture with $\eta=.01$, even $c=2560$ and 5,120 field observations leave every joint interval unresolved: the median width remains .10 although structural width is zero. At $\eta=.1$ and the same counts, field-only, calibration-only, and joint resolution rates are 98\%, 54\%, and 1\%, respectively. The joint median width is 0.07651. These examples expose calibration and weak-signal costs; they do not identify the unknown population channels of the earlier LLM experiment.

All cells and repetitions are released. Coverage and resolution denominators retain infeasible cases; width quantiles condition on feasibility. There are 84 infeasible intervals and four incorrect resolutions, all in the field-only ablation. Observed field-only coverage ranges from 96\% to 100\% across cells, with Monte Carlo errors and exact binomial intervals reported. All joint and calibration-only intervals contain the target in these runs. A 200/200 count has a two-sided 95\% binomial lower bound of approximately .982, so it is not evidence of perfect coverage. No simultaneous claim over cells is made.

\begin{figure*}[t]
\centering\includegraphics[width=.98\textwidth]{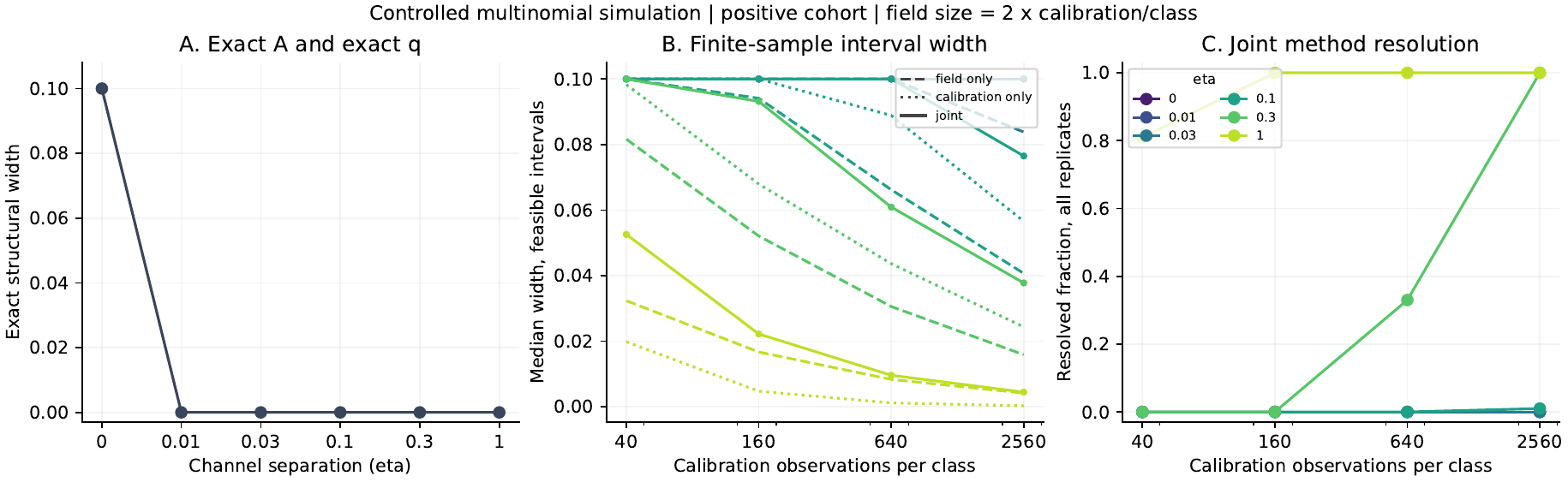}
\caption{Controlled separation of structural and statistical limitations. The figure displays the positive mixture; all three mixtures are in the released tables. Exact identification at every $\eta>0$ coexists with wide finite-sample intervals. Width summaries condition on feasibility; resolution denominators retain all repetitions. These are multinomial simulations, not additional LLM calls.}
\label{fig:precision}
\end{figure*}

\section{From telemetry to a decision receipt}
The accompanying offline viewer exposes the recorded audit to an analyst. Its selectors cover model, domain framing, field mixture, and logging condition. It displays observation counts, the known synthetic contrast, the interval, and its endpoint witness populations. The viewer renders precomputed LP results and their provenance; it neither silently reruns inference nor asks an LLM to recommend an investment. Its implementation is an inspectable prototype, with no measured usability or effect on analyst decisions.

The practical unit of measurement need not be a complete customer profile. For a declared future decision, it is enough to collect evidence about the corresponding contrast. This suggests a \emph{decision receipt}: a compact record of the measurement condition and the evidence the decision actually requires. We use this term for a proposed interface artifact, not an established standard or a claim of terminology priority.

A useful receipt distinguishes an observed action, preferences explicitly supplied to the agent, an agent's inference, and an outcome confirmed by the customer. Those sources have different evidentiary status. It records the agent and interface versions, relevant constraints, a missing-information state, and the specific decision for which the record is intended. A supplied preference does not become validated because an agent repeats it confidently. The receipt experiment tests only one minimal field, a reported leading attribute; it does not validate the entire proposed artifact.

Three design consequences follow. First, preserve the assigned measurement condition: pooling probe identities can erase the information gained from probing. Second, ask whether a cheaper measurement identifies the current contrast before trying to reconstruct all preferences. Third, revisit the receipt when the decision changes. Evidence sufficient for capacity versus flexibility need not identify willingness to pay, trust, or the value of an unrelated feature. Direct, consented outcome measurement may be both more informative and simpler than recovering a latent taxonomy. The research package includes a receipt specification separating these cases, with no claim of deployed customer effectiveness.

\section{Implications and limitations}
A valid interpretation starts with a specific product decision and outcome. Stable activity need not imply stable preference composition, and changed activity after an agent update need not imply changed demand. These are possible confoundings, not findings of widespread market failure. Calibration should record the agent, interface, prompt distribution, and time window.

The diagnostic is most useful when it makes the disagreement inspectable. A witness pair identifies which customer mixtures still fit the same declared evidence but favor different changes. That can motivate targeted, consented feedback. It does not prove that the proposed question is unbiased, cost-effective, or privacy preserving. An agent's explanation is another potentially useful observation channel, not ground truth about its customer.

Ordinary A/B testing remains valid for measured outcomes under its assumptions. The concern is substituting agent activity for unmeasured value. If a sound experiment directly measures the relevant customer outcome, this particular gap may already be resolved. Retention and support contacts can also supply information, subject to their own timing and selection issues.

The model assumes a declared finite taxonomy, stable within-class behavior, and a meaningful contrast. Missing classes, preference elicitation errors, selective feedback, correlated API calls, and channel drift can invalidate calibration transport. The proofs are conditional mathematical statements. The experiment uses two snapshots and one synthetic task family, so it cannot establish a general relationship between agent capability and product-learning value. Conservative intervals can be too wide to act on; confident point estimates can obscure unmeasured uncertainty. Both costs matter.

\paragraph{Human validation.} No human participants were recruited, and no simulated profile is counted as a person. The released offline instrument is an investigator preview. It labels its aid as a deterministic reference chooser and its exports as non-evidence. A customer study still requires a reviewed protocol, genuine recruitment, informed consent, an independently meaningful outcome measure, and actual responses. This paper's empirical claims are restricted to computation.

\section{Conclusion}
Successful execution and evidence for a future product decision are distinct objectives. A calibrated channel model makes that distinction testable for a declared contrast and logging scheme. Bounds and witness populations expose what the observations leave unresolved; uncertainty accounting prevents calibration estimates from becoming unwarranted certainty. The deterministic parser and controlled channels sharpen the practical lesson: preserve useful structured input, and distinguish structural ambiguity from statistical imprecision before changing the interface or collecting more logs. The released package provides proofs, actual model traces, controls, and reproducible analysis for this task. Extending its conclusions to customers requires customer evidence.

\paragraph{Availability.} Code, source, numerical records, and experiment provenance are at \url{https://github.com/shi1720/delegation-blind-spot}. The repository documents limitations and venue-specific requirements. This is a technical preprint, not a claim of peer review or acceptance.
\begingroup\raggedright\small
\bibliographystyle{plainnat}
\bibliography{references}

\begin{thebibliography}{13}
\providecommand{\natexlab}[1]{#1}
\providecommand{\url}[1]{\texttt{#1}}
\expandafter\ifx\csname urlstyle\endcsname\relax
  \providecommand{\doi}[1]{doi: #1}\else
  \providecommand{\doi}{doi: \begingroup \urlstyle{rm}\Url}\fi

\bibitem[Angelopoulos et~al.(2023)Angelopoulos, Bates, Fannjiang, Jordan, and
  Zrnic]{ppi2023}
Anastasios~N. Angelopoulos, Stephen Bates, Clara Fannjiang, Michael~I. Jordan,
  and Tijana Zrnic.
\newblock Prediction-powered inference.
\newblock \emph{Science}, 382\penalty0 (6671):\penalty0 669--674, 2023.
\newblock \doi{10.1126/science.adi6000}.
\newblock URL \url{https://arxiv.org/abs/2301.09633}.

\bibitem[Blackwell(1953)]{blackwell1953}
David Blackwell.
\newblock Equivalent comparisons of experiments.
\newblock \emph{The Annals of Mathematical Statistics}, 24\penalty0
  (2):\penalty0 265--272, June 1953.
\newblock \doi{10.1214/aoms/1177729032}.

\bibitem[Cherep et~al.(2026)Cherep, Ma, Xu, Shaked, Maes, and
  Singh]{cherep2025}
Manuel Cherep, Chengtian Ma, Abigail Xu, Maya Shaked, Pattie Maes, and Nikhil
  Singh.
\newblock A framework for studying {AI} agent behavior: Evidence from consumer
  choice experiments.
\newblock In \emph{International Conference on Learning Representations}, 2026.
\newblock URL \url{https://openreview.net/forum?id=LUrToUPS4x}.

\bibitem[Donoho(1994)]{donoho1994}
David~L. Donoho.
\newblock Statistical estimation and optimal recovery.
\newblock \emph{The Annals of Statistics}, 22\penalty0 (1):\penalty0 238--270,
  1994.
\newblock \doi{10.1214/aos/1176325367}.

\bibitem[Finkelstein et~al.(2021)Finkelstein, Adams, Saria, and
  Shpitser]{finkelstein2021}
Noam Finkelstein, Roy Adams, Suchi Saria, and Ilya Shpitser.
\newblock Partial identifiability in discrete data with measurement error.
\newblock In Cassio de~Campos and Marloes~H. Maathuis, editors,
  \emph{Proceedings of the Thirty-Seventh Conference on Uncertainty in
  Artificial Intelligence}, volume 161 of \emph{Proceedings of Machine Learning
  Research}, pages 1798--1808. PMLR, 2021.
\newblock URL \url{https://proceedings.mlr.press/v161/finkelstein21b.html}.

\bibitem[Kirschner et~al.(2023)Kirschner, Lattimore, and Krause]{kirschner2023}
Johannes Kirschner, Tor Lattimore, and Andreas Krause.
\newblock Linear partial monitoring for sequential decision making: Algorithms,
  regret bounds and applications.
\newblock \emph{Journal of Machine Learning Research}, 24\penalty0
  (346):\penalty0 1--45, 2023.
\newblock URL \url{https://jmlr.org/papers/v24/22-1248.html}.

\bibitem[Kops and Tsakas(2026)]{kops2026}
Christopher Kops and Elias Tsakas.
\newblock Choice via {AI}, February 2026.
\newblock URL \url{https://arxiv.org/abs/2602.04526}.
\newblock arXiv preprint.

\bibitem[Kraft and Larsen(2026)]{kraft2026}
Andreas Kraft and Poet Larsen.
\newblock Consumer preference transmission in agentic markets.
\newblock SSRN Working Paper 6864181, June 2026.
\newblock URL \url{https://ssrn.com/abstract=6864181}.

\bibitem[Lattimore and Szepesv{\'a}ri(2020)]{lattimore2020}
Tor Lattimore and Csaba Szepesv{\'a}ri.
\newblock \emph{Bandit Algorithms}.
\newblock Cambridge University Press, 2020.
\newblock \doi{10.1017/9781108571401}.
\newblock URL \url{https://tor-lattimore.com/downloads/book/book.pdf}.

\bibitem[Li et~al.(2025)Li, Zrnic, and Candes]{robust2025}
Puheng Li, Tijana Zrnic, and Emmanuel Candes.
\newblock Robust sampling for active statistical inference.
\newblock In D.~Belgrave, C.~Zhang, H.~Lin, R.~Pascanu, P.~Koniusz,
  M.~Ghassemi, and N.~Chen, editors, \emph{Advances in Neural Information
  Processing Systems}, volume 38, Main Conference, pages 68727--68756. Curran
  Associates, Inc., 2025.
\newblock \doi{10.52202/085713-2312}.
\newblock URL
  \url{https://proceedings.neurips.cc/paper_files/paper/2025/file/6389470564214983604d1ac81631c2c5-Paper-Conference.pdf}.

\bibitem[Lipton et~al.(2018)Lipton, Wang, and Smola]{lipton2018}
Zachary Lipton, Yu-Xiang Wang, and Alexander Smola.
\newblock Detecting and correcting for label shift with black box predictors.
\newblock In Jennifer Dy and Andreas Krause, editors, \emph{Proceedings of the
  35th International Conference on Machine Learning}, volume~80 of
  \emph{Proceedings of Machine Learning Research}, pages 3122--3130. PMLR,
  2018.
\newblock URL \url{https://proceedings.mlr.press/v80/lipton18a.html}.

\bibitem[Suleymanov(2026)]{suleymanov2026}
Elchin Suleymanov.
\newblock A revealed preference framework for {AI} alignment, March 2026.
\newblock URL \url{https://arxiv.org/abs/2603.27868}.
\newblock arXiv preprint.

\bibitem[Zrnic and Candes(2024)]{asi2024}
Tijana Zrnic and Emmanuel Candes.
\newblock Active statistical inference.
\newblock In Ruslan Salakhutdinov, Zico Kolter, Katherine Heller, Adrian
  Weller, Nuria Oliver, Jonathan Scarlett, and Felix Berkenkamp, editors,
  \emph{Proceedings of the 41st International Conference on Machine Learning},
  volume 235 of \emph{Proceedings of Machine Learning Research}, pages
  62993--63010. PMLR, 2024.
\newblock URL \url{https://proceedings.mlr.press/v235/zrnic24a.html}.

\end{thebibliography}
\endgroup
\clearpage
\appendix
\onecolumn
\section{Proofs and extensions}
\label{app:proofs}
\label{sec:proofs}

The results below make the observation assumptions explicit and provide
verifiable foundations for the implementation. They are applications of
established identification, convex duality, concentration, and decision-theoretic
arguments, not claims of new general mathematical principles. Relevant
antecedents include partial monitoring \citep{kirschner2023}, partial
identification under measurement error \citep{finkelstein2021}, optimal
recovery \citep{donoho1994}, and comparison of experiments
\citep{blackwell1953}.

Let $A\in\mathbb R^{m\times K}$ be nonnegative and column-stochastic. Write
$\mathcal S_K=\{p\in\mathbb R^K:p\geq0,\ \mathbf1^\top p=1\}$,
$q=Ap$, and $\Delta(p)=d^\top p$ for a fixed known vector $d\in\mathbb R^K$.
For $q\in A\mathcal S_K$, define
\[
 \mathcal P_q=\{p\in\mathcal S_K:Ap=q\},\qquad
 L(q)=\min_{p\in\mathcal P_q}d^\top p,\quad
 U(q)=\max_{p\in\mathcal P_q}d^\top p.
\]
These extrema are attained because $\mathcal P_q$ is nonempty and compact.
The target is a population contrast. It is not recovery of individual private
preferences.

\begin{theorem}[Global and local identification]
\label{thm:global-identification}
The contrast $d^\top p$ is identified for every feasible $q$ if and only if
$d\in\operatorname{row}(A)$. At a particular $q$, it is identified if and only if
$d\perp\operatorname{span}(\mathcal P_q-\mathcal P_q)$.
\end{theorem}
\begin{proof}
If $d=A^\top v$, then $d^\top p=v^\top q$ for every compatible population.
Conversely, if $d\notin\operatorname{row}(A)$, there is an $h\in\ker(A)$ with
$d^\top h\ne0$. Column normalization gives $\mathbf1^\top h=0$.
Choose an interior $p_0\in\mathcal S_K$ and sufficiently small $t>0$ so that
$p_0\pm th\geq0$. Both populations belong to the simplex, produce $Ap_0$, and
have different contrasts. The local statement follows because constancy of
$d^\top p$ on $\mathcal P_q$ is equivalent to orthogonality to every pairwise
difference, hence to their span.
\end{proof}

The global condition need not hold at an identified boundary point. For
example, the channel with columns $(1,0)^\top,(0,1)^\top,(0,1)^\top$ cannot
generally separate classes two and three, but $q=(1,0)^\top$ uniquely determines
$p=(1,0,0)^\top$. This distinction prevents a rank-only test from replacing the
feasible-set calculation. Observation-span conditions have direct antecedents
in partial monitoring \citep{kirschner2023}.

\begin{theorem}[Magnitude of the globally hidden contrast]
\label{thm:hidden-width}
Define
\[
 W(A,d)=\max_{p,r\in\mathcal S_K:Ap=Ar}|d^\top(p-r)|.
\]
Then
\begin{equation}
\label{eq:hidden-width-dual}
 W(A,d)
 =\max_{Ah=0,\ \|h\|_1\leq2}d^\top h
 =2\min_{v\in\mathbb R^m}\|d-A^\top v\|_\infty.
\end{equation}
\end{theorem}
\begin{proof}
Every $h=p-r$ in the first optimization satisfies $Ah=0$ and $\|h\|_1\leq2$.
Conversely, $Ah=0$ implies $\mathbf1^\top h=0$. Let $h_+,h_-$ be its positive
and negative parts, with common mass $s=\|h\|_1/2\leq1$. For any
$u\in\mathcal S_K$, the vectors $p=h_++(1-s)u$ and $r=h_-+(1-s)u$ are
compatible simplex points and satisfy $p-r=h$. The feasible set is symmetric,
so maximizing the absolute contrast equals maximizing the signed contrast.

For the second equality, consider
\[
 \min_{v,t}\ t\quad\text{subject to}\quad
 d-A^\top v\leq t\mathbf1,\quad
 A^\top v-d\leq t\mathbf1,\quad t\geq0.
\]
Assign nonnegative multipliers $a,b$ to these two inequalities. Minimizing the
Lagrangian over unrestricted $v$ requires $A(a-b)=0$; minimizing over $t\geq0$
requires $\mathbf1^\top(a+b)\leq1$. The dual objective is $d^\top(a-b)$.
Writing $z=a-b$ gives precisely $Az=0,\ \|z\|_1\leq1$: the reverse direction
uses $a=z_+,b=z_-$. Primal feasibility follows by taking $v=0$ and
$t=\|d\|_\infty$; the primal is bounded below, so finite-dimensional LP strong
duality applies. Rescaling $h=2z$ proves the identity.
\end{proof}

$W(A,d)$ is a worst-case width over all possible $q$; the interval at an observed
$q$ may be narrower. The dual quantity is a standard distance to an observable
subspace, with conceptual antecedents in optimal recovery
\citep{donoho1994}. It is not, by itself, a measured market effect.

\begin{theorem}[Irreducible estimation and decision loss]
\label{thm:decision-loss}
Fix $q$ and write $L=L(q)$, $U=U(q)$, and $w=U-L$. Suppose the data consist of
any finite number of iid actions from $q$, together with independent analyst
randomization. Every estimator has worst-case mean absolute error at least
$w/2$ and worst-case mean squared error at least $w^2/4$ over $\mathcal P_q$.
If $q$ is supplied exactly, the midpoint attains both bounds.

If $L<0<U$, the minimax customer-value regret for choosing between variants zero
and one, when $q$ is known exactly, is
\begin{equation}
\label{eq:minimax-regret}
 R^*(q)=\frac{-LU}{U-L}.
\end{equation}
It is attained by choosing variant one with probability $t^*=U/(U-L)$.
The minimax regret among deterministic decisions is $\min\{-L,U\}$. If the
interval does not strictly straddle zero, minimax regret is zero.
\end{theorem}
\begin{proof}
All compatible populations yield the same data distribution. Let $T$ have the
common distribution of an estimator at populations attaining the endpoints.
Pointwise,
\[
 |T-L|+|T-U|\geq w,\qquad
 \frac{(T-L)^2+(T-U)^2}{2}
 =\left(T-\frac{L+U}{2}\right)^2+\frac{w^2}{4}.
\]
Taking expectations bounds the larger endpoint risk from below. The constant
midpoint attains both lower bounds over the whole interval when $q$ is known.

For a randomized decision rule let $t$ be its common probability of selecting
variant one. Its regrets at the lower and upper endpoints are $(-L)t$ and
$U(1-t)$, respectively; intermediate contrasts give no larger regret. Minimizing
$\max\{(-L)t,U(1-t)\}$ over $t\in[0,1]$ equates the terms and gives
\eqref{eq:minimax-regret}. Restricting $t$ to zero or one gives the deterministic
result. If the interval lies weakly on one side of zero, a single variant is
weakly optimal throughout.
\end{proof}

For strict opposite-sign endpoints, the decision error probabilities are $t$
and $1-t$, so their equally weighted average is $1/2$. These are fixed-channel
indistinguishability statements, a standard lower-bound argument
\citep{lattimore2020}. The attainability claims presume exact $q$; finite-sample
uncertainty can add loss. A probe that changes the observation channel can
break the indistinguishability.

\begin{theorem}[Coarsening]
\label{thm:coarsening}
Let $G$ be column-stochastic and set $B=GA$. For every feasible $q$,
\[
 \mathcal P(A,q)\subseteq\mathcal P(B,Gq).
\]
The interval under $B$ at $Gq$ therefore contains the interval under $A$ at $q$,
and $W(B,d)\geq W(A,d)$.
\end{theorem}
\begin{proof}
$Ap=q$ implies $Bp=GAp=Gq$. Minimization over the enlarged set cannot increase
the lower bound, and maximization cannot decrease the upper bound. Every pair
with $Ap=Ar$ also has $Bp=Br$, proving the global inequality.
\end{proof}

This finite-channel implication belongs to the comparison-of-experiments
viewpoint \citep{blackwell1953}. It does not establish that a more competent
agent is a coarsening of another agent. That is an additional, testable
relationship, not an assumption licensed by higher task success.

\begin{theorem}[A bias-aware finite-sample certificate]
\label{thm:linear-certificate}
Let $Y_1,\ldots,Y_n$ be iid categorical actions from $q=Ap$. Choose
$v\in\mathbb R^m$ independently of these data and define
$\epsilon(v)=\|d-A^\top v\|_\infty$ and
$\operatorname{rng}(v)=\max_a v_a-\min_a v_a$.
With probability at least $1-\alpha$,
\begin{equation}
\label{eq:linear-certificate}
 \left|\frac1n\sum_{i=1}^n v_{Y_i}-d^\top p\right|
 \leq \epsilon(v)+\operatorname{rng}(v)
       \sqrt{\frac{\log(2/\alpha)}{2n}}.
\end{equation}
If the deployed channel is $A'$ and
$\max_k\operatorname{TV}(A'_{\cdot k},A_{\cdot k})\leq\tau$, the bound holds
for iid data from $A'p$ after adding $\tau\operatorname{rng}(v)$.
\end{theorem}
\begin{proof}
The mean of $v_{Y_i}$ under $A$ is $v^\top Ap$. Since $p$ is a probability
vector, its difference from $d^\top p$ has absolute value at most
$\|A^\top v-d\|_\infty$. Hoeffding's inequality for variables in
$[\min_a v_a,\max_a v_a]$ bounds the sampling deviation by the second term
with probability at least $1-\alpha$. The triangle inequality proves
\eqref{eq:linear-certificate}.

For distributions $r,s$, write $b=\min_a v_a$ and use
$\sum_a(r_a-s_a)=0$ to obtain
$|v^\top(r-s)|\leq\operatorname{rng}(v)\operatorname{TV}(r,s)$.
Applying this bound to every channel column and averaging over $p$ bounds the
additional drift bias by $\tau\operatorname{rng}(v)$.
\end{proof}

Minimizing the right-hand side over $v$ using only $A,d,n,\alpha,\tau$ preserves
the guarantee. Choosing $v$ against the same field observations requires a
separate argument. An estimated channel is not automatically a known channel;
its uncertainty must be covered by a valid drift bound or the joint confidence
set. Bias-aware linear estimation and concentration are established tools
\citep{donoho1994,lattimore2020}.

\begin{theorem}[Identification from logged probes]
\label{thm:logged-probes}
Let probe $e\in\{1,\ldots,E\}$ be assigned independently of latent class with
probability $\rho_e>0$, where $\sum_e\rho_e=1$, and let its channel be $A_e$.
Suppose the probe identity is recorded and all probes share the same population
composition $p$. The joint probe/action channel is
$B=\operatorname{vstack}(\rho_1A_1,\ldots,\rho_EA_E)$, and
\[
 d^\top p\text{ is globally identified}
 \quad\Longleftrightarrow\quad
 d\in\operatorname{row}(B)
 =\operatorname{span}\!\left(\bigcup_e\operatorname{row}(A_e)\right).
\]
\end{theorem}
\begin{proof}
$B$ is column-stochastic and
$\ker(B)=\bigcap_e\ker(A_e)$ because every $\rho_e$ is positive. Apply
Theorem~\ref{thm:global-identification} and orthogonal-complement identities.
\end{proof}

If probe labels are discarded, the channel is instead $\sum_e\rho_e A_e$.
For example, $A_1=I_2$ and $A_2$ equal to $I_2$ with its rows exchanged are each
fully revealing. At equal assignment probabilities their unlogged average has
identical columns $(1/2,1/2)^\top$. Logged experimental variation can therefore
be informative when unlogged variation is not. Probe choice and its costs
connect directly to experimental design and partial monitoring
\citep{kirschner2023,lattimore2020}.

\begin{proposition}[Continuous cost allocation for fixed weights]
\label{prop:cost-allocation}
Suppose independent samples from probe $e$ have size $n_e>0$, and fixed weights
$v_e$ satisfy $\sum_e A_e^\top v_e=d$. Write
$\sigma_e^2=\operatorname{Var}_{A_ep}(v_{e,Y})$ and assume $\sigma_e>0$.
The estimator $\sum_e n_e^{-1}\sum_i v_{e,Y_{ei}}$ is unbiased, with variance
$\sum_e\sigma_e^2/n_e$. Under positive costs $c_e$ and the continuous budget
$\sum_e c_en_e=C$, its minimum variance and optimal allocation are
\[
 V_{\min}=\frac{(\sum_e\sigma_e\sqrt{c_e})^2}{C},\qquad
 n_e^*=\frac{C\sigma_e/\sqrt{c_e}}{\sum_j\sigma_j\sqrt{c_j}}.
\]
\end{proposition}
\begin{proof}
The weight identity gives unbiasedness. Independence gives the variance.
Cauchy--Schwarz yields
\[
 (\sum_e\sigma_e\sqrt{c_e})^2
 \leq (\sum_e\sigma_e^2/n_e)(\sum_e c_en_e).
\]
Equality holds at the displayed allocation.
\end{proof}
This familiar calculation is not a new acquisition algorithm. Integer sample
sizes, unknown variances, adaptive weights, or required minimum allocations
need additional treatment. Zero variances require the corresponding limiting
allocation or explicit minimum-count constraints.

\begin{proposition}[Rectangular uncertainty and coverage]
\label{prop:uncertainty-coverage}
Suppose channel columns lie in coordinate boxes
$\underline A_{ak}\leq A_{ak}\leq\overline A_{ak}$, with $0\leq\underline A\leq\overline A\leq1$ and each containing at least
one simplex vector. Frequencies lie in
$\underline q\leq q\leq\overline q$. Introduce $J_{ak}$ and impose
\begin{align}
 p&\geq0,\quad J\geq0, & \mathbf1^\top p&=1, \nonumber\\
 \sum_a J_{ak}&=p_k, &
 \underline A_{ak}p_k\leq J_{ak}&\leq\overline A_{ak}p_k,\label{eq:joint-lp}\\
 \underline q_a&\leq\sum_kJ_{ak}\leq\overline q_a. &&\nonumber
\end{align}
Minimizing and maximizing $d^\top p$ over these constraints gives attainable
bounds within the supplied rectangular model. If the channel and frequency
boxes cover their true values with probabilities at least $1-\alpha_A$ and
$1-\alpha_q$, respectively, the interval covers $d^\top p$ with probability
at least $1-\alpha_A-\alpha_q$.
\end{proposition}
\begin{proof}
Every admissible $(A,p)$ maps to $J=A\operatorname{diag}(p)$ satisfying
\eqref{eq:joint-lp}. Conversely, if $p_k>0$, define $A_{ak}=J_{ak}/p_k$.
The resulting column sums to one and satisfies its box constraints. If
$p_k=0$, the constraints force the entire column of $J$ to zero; choose any
simplex vector from that column's nonempty channel box. These choices produce
an admissible channel with frequencies $\sum_kJ_{\cdot k}$. This establishes
exactness, including zero-mass columns.

On the intersection of the two coverage events, the true pair
$(p,A\operatorname{diag}(p))$ is feasible. Its contrast lies between the extrema.
The union bound gives the asserted probability without requiring independence
of the confidence events.
\end{proof}

This confidence-set propagation is already present in the measurement-error
literature, including the supplementary Proposition~1 of
\citet{finkelstein2021}. An infeasible program is a failure state, not a narrow
confidence interval. Fixed-sample intervals do not justify optional stopping.

\paragraph{Uncertain outcomes.}
If $d$ lies in an independent coordinate rectangle
$[\underline d,\overline d]$, sharp rectangular-model endpoints are
$\min_p\underline d^\top p$ and $\max_p\overline d^\top p$ over the same
feasible set. This follows because $p\geq0$ makes the coordinate endpoints
optimal for each fixed $p$. If this outcome rectangle has error probability
$\alpha_d$, the same proof gives coverage at least
$1-\alpha_A-\alpha_q-\alpha_d$. Parameter correlations excluded by the
rectangle can make these bounds conservative.

\paragraph{Omitted classes.}
A separate contamination model is needed if some customers lie outside the
calibrated taxonomy. Let their mass be $\eta\in[0,\bar\eta]$, their total action
masses be $u_a\geq0$, and known-class masses be $z_k\geq0$. Replace the mass and
frequency constraints by
\[
 \sum_k z_k=1-\eta,\quad \sum_a u_a=\eta,\quad
 \underline q_a\leq\sum_kJ_{ak}+u_a\leq\overline q_a,
\]
and replace $p$ by $z$ in the channel constraints. If unknown-class contrasts
lie in $[-D,D]$, introduce $t$ with $-D\eta\leq t\leq D\eta$ and optimize
$d^\top z+t$. These constraints are linear. Every allowed contaminated model
maps to these masses. Conversely, normalize nonzero known-class columns as in
the proof above; when $\eta>0$, one unknown class with channel $u/\eta$ and
contrast $t/\eta$ realizes the residual. At $\eta=0$, both residuals vanish.
Thus the extension is exact for this broad contamination model. Merely widening
an interval computed from contaminated frequencies without modifying the
observation constraints does not establish such protection.

\begin{proposition}[Calibration on a fixed heterogeneous panel]
\label{prop:heterogeneous-panel}
For class $k$, let $Y_{k1},\ldots,Y_{kn_k}$ be independent responses to a
prespecified, possibly heterogeneous prompt panel. Define
\[
 \bar A_{ak}=\frac1{n_k}\sum_i\Pr(Y_{ki}=a),\qquad
 \widehat A_{ak}=\frac1{n_k}\sum_i\mathbf1\{Y_{ki}=a\}.
\]
For $m$ categories and $K$ classes, simultaneous coordinate intervals with
radius
\[
 r_k=\sqrt{\frac{\log(2mK/\alpha_A)}{2n_k}}
\]
cover $\bar A$ with probability at least $1-\alpha_A$. Endpoints may be clipped
to $[0,1]$.
\end{proposition}
\begin{proof}
For fixed $a,k$, the indicators are independent variables in $[0,1]$, though
they need not have identical expectations. Hoeffding's inequality gives
\[
 \Pr(|\widehat A_{ak}-\bar A_{ak}|>r_k)
 \leq 2e^{-2n_kr_k^2}=\frac{\alpha_A}{mK}.
\]
A union bound over all coordinates proves simultaneous coverage. Clipping to
the probability range cannot exclude a true coordinate already covered.
\end{proof}

The result targets the fixed-panel average channel, not an arbitrary deployment
channel. Matching prompt composition or an explicit transport bound is still
required. Alternatively, independently sampled prompts from a prespecified
distribution, combined with independent stable responses, support the iid
multinomial interpretation used by the benchmark's calibration intervals.
Refusals and errors must remain declared categories or enter a missing-data
model. Dependence across calls or adaptive prompt selection is not covered by
the elementary panel argument. Simulated profiles remain synthetic experimental
inputs, irrespective of model fluency.

\section{Exploratory audit-estimator pilot}\label{sec:pilot}
Before the model experiment, we evaluated established audit-corrected estimators on fixed synthetic pools of 4,000 records. Five scenarios, four expected audit budgets, and five sampling policies each received 1,000 independent audit replicates, yielding 100,000 estimator evaluations. These are numerical evaluations, not model calls or people. Expected audit budgets are matched; independent Bernoulli selection makes realized counts random and their distribution is recorded.

The estimator adds an inverse-probability weighted residual correction to proxy predictions. Uniform, historical-error active, fixed-mixture, and robust linear-path sampling use observable information; a separate oracle accesses unavailable outcome residuals. The robust-path policy adapts \citet{robust2025} to a declared box uncertainty set. It is not a new method or a full reproduction of that paper's experiments.

In the deliberately constructed strong hidden-shift case, the target contrast changes from $.0065$ to $-.104$ while visible records and the proxy contrast $.0105$ stay fixed. At 100 expected audits, historical active sampling has RMSE $.0663$ and nominal 95\% Wald coverage 88.9\%; uniform sampling has $.0577$ and 93.9\%; the robust-path adaptation has $.0558$ and 93.4\%. Coverage Monte Carlo standard errors are approximately 1.0, .8, and .8 percentage points. Conservative Bernstein intervals cover in all recorded runs but are often uninformative. Point-estimator unbiasedness does not guarantee adequate finite-sample Wald coverage. Two complete runs reproduce all four numerical CSV outputs byte for byte.

\begin{figure}[h]
\centering\includegraphics[width=.96\textwidth]{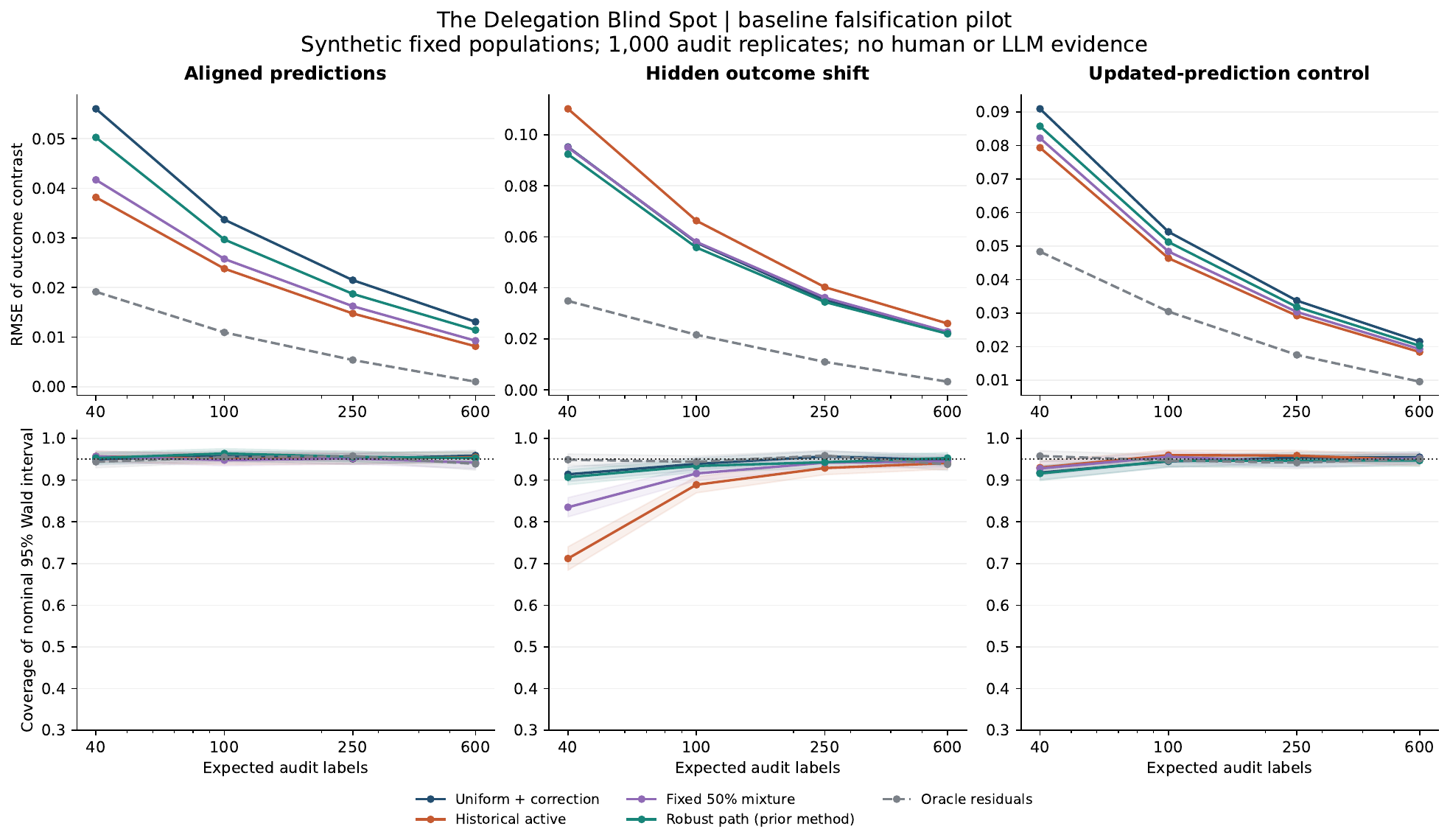}
\caption{Exploratory synthetic audit pilot. Error panels use different scales; oracle information is privileged. No claim of an effect in real delegated behavior follows from this construction.}
\end{figure}

\section{Artifact and claim audit}
\begin{center}\small
\begin{tabular}{p{.24\textwidth}p{.31\textwidth}p{.35\textwidth}}
\toprule Claim & Evidence & Limit \\\midrule
Decision-specific identification & Theorems, LP witnesses, independent numerical checks & Standard mathematics under a finite channel model \\
Sampling and drift certificates & Full proofs, exhaustive small examples & Conditional on sampling, taxonomy, and drift assumptions \\
Named model behavior & Frozen prompts and raw API records & Synthetic utilities, two model snapshots \\
Product-decision uncertainty & Held-out synthetic mixtures and declared logging schemes & No conclusion about every possible telemetry schema \\
Customer validity & Protocol and preview instrument only & No completed human study \\
\bottomrule
\end{tabular}
\end{center}
The source package includes pinned scientific dependencies, tests, exact preparation and analysis commands, checksums, an MIT license, and a citation file. Regenerating an API request does not guarantee an identical remote response. Offline analysis reproduces figures from the recorded responses. Public traces contain constructed inputs only and exclude credentials.
\end{document}